\documentclass{article}

\usepackage{amsmath,amsfonts,amssymb,times,graphicx,natbib,algorithm,algorithmic,hyperref}

\usepackage[accepted]{whi2017}

\icmltitlerunning{Interpreting Classifiers through Attribute Interactions in Datasets}

\usepackage[linesnumbered,algo2e]{algorithm2e}
\usepackage{tabularx}
\usepackage{colortbl}
\usepackage{dashrule}
\usepackage{array}
\usepackage{amsmath}
\usepackage{amssymb}
\usepackage{amsthm}
\usepackage[usenames,dvipsnames,svgnames,table]{xcolor}
\usepackage{graphicx}
\usepackage{makeidx}
\usepackage{booktabs} 
\usepackage{hyperref}

\usepackage{hyphenat}
\usepackage{xspace}

\usepackage{subcaption}
\usepackage{calc}
\usepackage{units}

\protect\newcommand{\set}[1]{\ensuremath{\left\{#1\right\}}}
\newcommand{\prob}[1]{\ensuremath{P\left(#1\right)}}
\newcommand{\attr}[1]{\texttt{\small{#1}}}

\newcommand{\astrid}{\textsc{astrid}\xspace}

\newcommand{\nb}{na\"ive Bayes\xspace}

\newtheorem{definition}{Definition}
\newtheorem{lemma}{Lemma}
\newtheorem{problem}{Problem}

\newcommand{\squishlist}{
 \begin{list}{$\bullet$}
  { \setlength{\itemsep}{0pt}
     \setlength{\parsep}{3pt}
     \setlength{\topsep}{3pt}
     \setlength{\partopsep}{0pt}
     \setlength{\leftmargin}{1.5em}
     \setlength{\labelwidth}{1em}
     \setlength{\labelsep}{0.5em} } }
\newcommand{\squishend}{
  \end{list}  }

\definecolor{cbluelight}{HTML}{ccd8e4}

\begin{document}

\twocolumn[
\icmltitle{Interpreting Classifiers through Attribute Interactions in Datasets}



\begin{icmlauthorlist}
\icmlauthor{Andreas Henelius}{to}
\icmlauthor{Kai Puolam{\"a}ki}{to}
\icmlauthor{Antti Ukkonen}{to}
\end{icmlauthorlist}

\icmlaffiliation{to}{Finnish Institute of Occupational Health, Helsinki, Finland}

\icmlcorrespondingauthor{Andreas Henelius}{andreas.henelius@ttl.fi}

\icmlkeywords{attribute interactions, constrained randomisation, classifiers, interpretability}

\vskip 0.3in
]

\printAffiliationsAndNotice{}

\begin{abstract} 
In this work we present the novel \astrid method for investigating which attribute interactions classifiers exploit when making predictions. Attribute interactions in classification tasks mean that two or more attributes together provide stronger evidence for a particular class label. Knowledge of such interactions makes models more interpretable by revealing associations between attributes. This has applications, e.g., in pharmacovigilance to identify interactions between drugs or in bioinformatics to investigate associations between single nucleotide polymorphisms. We also show how the found attribute partitioning is related to a factorisation of the data generating distribution and empirically demonstrate the utility of the proposed method.
\end{abstract} 

\section{Introduction}
\label{sec:introduction}
A lot of attention has been on creating high-performing classifiers
such as, e.g., support vector machines (SVMs) \citep{cortes:1995:a}
and random forest \citep{breiman:2001:a}, both of which are among the
best-performing classifiers \citep{delgado:2014:a}. However, the
complexity of many state-of-the-art classifiers means that they are
essentially opaque, black boxes, i.e., it is very difficult to gain
insight into how the classifiers work. Gaining insight into machine
learning models is a topic that will become more important in the
future, e.g., due to possible legislative requirements
\citep{goodman:2016:a}. Interpretability of machine learning models is
a multifaceted problem, one aspect of which is post-hoc
interpretability \citep{lipton:2016:a}, i.e., gaining insight into how
the method reaches the given predictions.

Interpreting black box machine learning models in terms of
\emph{attribute interactions} provides one form of post-hoc
interpretability and is the focus of this paper. Given a supervised
classification dataset $D = \left(X, C \right)$, where $X$ is a data
matrix with $m$ predictor attributes $x_1, \ldots, x_m$ (e.g., gender,
age etc), and $C$ is a vector with a target attributes (class), an
interaction between a subset of these $m$ attributes means that the
attributes together provide stronger evidence concerning $C$ than if
the attributes are considered alone. We say that attributes interact
whenever they are \emph{conditionally dependent given the class}. We
next motivate attribute interactions from the perspective of
interpretability of real-world problems.

Two difficult problems involving interactions concern drug-drug
interactions in pharmacovigilance \citep[e.g., ][]{zhang:2017:a,
  cheng:2014:a} and investigating interactions between single
nucleotide polymorphisms (SNPs) in bioinformatics \citep[e.g.,
][]{lunetta:2004:a, moore:2010:a}. Recently, machine learning methods
have been applied to investigate drug-drug \citep{henelius:2015:gepp}
and gene-gene interactions \citep{li:2016:a}. The benefit of using
powerful classifiers, such as random forest, is that one does not need
to specify the exact form of interactions between attributes
\citep{lunetta:2004:a}, which is necessary in many traditional
statistical methods (e.g., linear regression models that include
interaction terms). To utilise classifiers in this manner for studying
associations in the data requires that we have some method for
revealing \emph{how the classifier perceives attribute interactions}.

A \emph{grouping} of the attributes in a dataset is a partition where
interacting attributes are in the same group, while non-interacting
(i.e., independent) attributes are in different groups. In this paper
we study two problems. Firstly we want to \emph{determine if a
  particular grouping of attributes represents the attribute
  interaction structure in a given dataset}. Secondly, we want to
\emph{automatically find a maximum cardinality grouping of the
  attributes in a given dataset}.

We approach these problems using the following intuition concerning
classifiers, which are used as tools to investigate interactions. A
classifier tries to model the class probabilities given the data,
i.e., the probability $\prob{C \mid X} \propto \prob{X \mid C }
\prob{C}$. Here $P \left( X \mid C \right)$ is the
\emph{class-conditional} distribution of the attributes, which we
focus on here. Formally, let $\mathcal{S}$ represent a factorisation
of $P \left( X \mid C \right)$ into independent factors, i.e.,
\begin{equation}
  \label{eq:P}
  P \left( X \mid C ; \mathcal{S}\right) =
 \prod_{S\in \mathcal{S}} \prob{X \left(\cdot, S \right) \mid C }
\end{equation}
where $X \left(\cdot, S \right) $ only contains the attributes in the
set $S$. In other words, interacting attributes are in the same group
$S \in \mathcal{S}$ and, hence, in the same factor in $P \left( X \mid
C ; \mathcal{S}\right)$.

Assume that the dataset $D$ is sampled from a factorised distribution
of the form given in Eq.~\eqref{eq:P} for some $\mathcal{S}$. Further
assume that we can generate datasets $D^\mathcal{S}$ that are
exchangeable with $D$. Suppose now that we train a classifier $f_1$
using $D$ and that we train a second classifier $f_2$ (of the same
type as $f_1$) using $D^\mathcal{S}$. Now, if classifiers $f_1$ and
$f_2$ cannot be distinguished from each other in terms of accuracy on
the same test data, it means that the factorisation $\mathcal{S}$
captures the class-dependent structure in the data to the extent
needed by the classifier.  On the other hand, if $f_2$ performs worse
than $f_1$, some essential relationships in the data needed by the
classifier are no longer present, i.e., $D$ has not been sampled from
a distribution of the form given by Eq.~\eqref{eq:P}. To determine
whether $f_1$ and $f_2$ are indistinguishable, we compute a confidence
interval (CI) for the performance of $f_2$ by generating an ensemble
of datasets $D^\mathcal{S}$. If the performance of $f_1$ is above the
CI we conclude that the factorisation $\mathcal{S}$ is not valid.

\subsection{Related Work}
In this paper we combine the probabilistic approach of
\citet{ojala2010jmlr} studying whether a classifier utilises attribute
interactions at all with the method of \citet{henelius:2014:peek}
allowing identification of groups of interacting attributes. For a
review on attribute interactions in data mining see, e.g.,
\citet{freitas:2001:a}. Interactions have been considered in feature
selection \citep{zhao:2007:a, zhao:2009:a}. \citet{mampaey:2013:a}
partition attributes by a greedy hierarchical clustering algorithm
based on Minimum Description Length (MDL). Their goal is similar to
our, but we focus on supervised learning. \citet{tatti:2011:a} ordered
attributes according to their dependencies while
\citet{jakulin:2003:a} quantified the degree of attribute interaction
and \citet{jakulin:2004:a} factorised the joint data distribution and
presented a method for significance testing of attribute interactions.


\subsection{Contributions}
We present and study the two problems of (i) assessing whether a
particular grouping of attributes represents the class-conditional
structure of a dataset (Sec. \ref{sec:1group}) and (ii) automatically
discovering the attribute grouping of highest granularity
(Sec. \ref{sec:automatically}). We empirically demonstrate using
synthetic and real data how the proposed \astrid \footnote{R-package
  available:\url{https://github.com/bwrc/astrid-r}}
(\textsc{a}utomatic \textsc{str}ucture \textsc{id}entification) method
finds attribute interactions in data
(Secs. \ref{sec:experiments}--\ref{sec:discussion}).


\section{Methods}
In this section we consider (i) how to determine if a particular
attribute grouping is a valid factorisation of the class-conditional
joint distribution, and (ii) automatically finding the maximum
cardinality attribute grouping.

\subsection{Preliminaries}
Let $X$ be an $n \times m$ data matrix, where $X(i,\cdot)$ denotes the
$i$th row (item), $X(\cdot,j)$ the $j$th column (attribute) of $X$,
and $X(\cdot,S)$ the columns of $X$ given by $S$, where $S\subseteq
[m]=\{1,\ldots,m\}$, respectively.  Let $\mathcal{C}$ be a finite set
of class labels and let $C$ be an $n$-vector of class labels, such
that $C\left(i\right)$ gives the class label for $X(i,\cdot)$. We
denote a dataset $D$ by the tuple $D=\left(X,C\right)$.

We denote by $\mathcal{P}$ the set of disjoint partitions of
$[m]=\{1,\ldots,m\}$, where a partition $\mathcal{S}\in\mathcal{P}$
satisfies $\cup_{S\in\mathcal{S}}{S}=[m]$ and for all $S,S'\in\mathcal{S}$
either $S=S'$ or $S\cap S'=\emptyset$, respectively.

Here we assume that the dataset has been sampled i.i.d., i.e., the dataset
$D$ follows a joint probability distribution given by
\begin{equation}
  \label{eq:jdfull}
  \begin{array}{lcl}
\prob{D} &=&
\prod_{i\in[n]}{P(X\left(i,\cdot\right),C\left(i\right))}\\
&=&\overbrace{\prod_{i \in \left\lbrack n \right\rbrack} \prob{X\left(i,\cdot \right) \mid C(i)}}^{\prob{X  \mid C}} \prob{C \left( i \right)},
  \end{array}
\end{equation}
where $\prob{X\mid C}$ is the \emph{class-conditional
  distribution}. We consider a factorisation of $\prob{D}$ into
class-conditional factors given by the grouping $\mathcal{S}\in\mathcal{P}$ and write
\begin{equation}
\label{eq:jdfact}
  \prob{D} = \overbrace{\prod_{i \in \left\lbrack n \right\rbrack} \prod_{S\in \mathcal{S}} \prob{X \left(i, S \right) \mid C \left( i \right)}}^{\prod_{S\in\mathcal{S}}{\prob{X(\cdot,S)  \mid C}}} \prob{C \left( i \right)} .
\end{equation}
Given an observed dataset $D$, we want to find the attribute
associations in the data
and ask: \emph{Has the observed dataset $D$ been sampled from a
  distribution given by Eq.~\eqref{eq:jdfact} with the grouping given
  by $\mathcal{S}\in\mathcal{P}$?}

\subsection{Framework for Investigating Factorisations}
\label{sec:1group}
Our goal is to determine whether the data obeys the factorised
distribution of Eq.~\eqref{eq:jdfact}. To do this we compare the
accuracy of a classifier trained using the original data with the
confidence interval (CI) formed from the accuracies of a collection of
classifiers trained using permuted data. The permuted datasets are
formed such that they are exchangeable with the original dataset if
Eq.~\eqref{eq:jdfact} holds. If the accuracy of the original data is
above the CI we can conclude with high confidence that the data does
not obey the factorised distribution.

We denote a classifier trained using the dataset $D$ by $f_D$. Further
assume that we have a separate independent test dataset from the same
distribution as $D$, denoted by
$D_\mathrm{test}=\left(X_\mathrm{test},C_\mathrm{test}\right)$.
\begin{definition}\emph{Classification Accuracy}
  \label{def:teststatistic}
  Given the above definitions, the accuracy for a classifier trained
  using $D$ is given by
  \begin{equation}
    \label{eq:teststatistic}
    T\left(D\right)=\frac 1{n_\mathrm{test}}
    \sum_{i=1}^{n_\mathrm{test}}{ I\left\lbrack
      f_D\left(X_\mathrm{test}\left(i,\cdot\right)\right)=C_\mathrm{test}\left(i\right)
      \right\rbrack},
  \end{equation}
  where $I\left\lbrack \Box \right\rbrack$ is the indicator function
  and $n_\mathrm{test}$ is the number of items in the test dataset.
\end{definition}
Note that $T$ is not the accuracy of $f$ on $D$, but the accuracy of
$f$ on $X_\mathrm{test}$ when $f$ is trained using $D$. Because direct
sampling from Eq.~\eqref{eq:jdfact} is not possible as the data
generating model is unknown, we generate
the permuted data matrices $X^\mathcal{S}$ (defined below)
so that they have
same probability as $X$ \emph{under the assumption that $X$ is a
  sample from a factorised distribution as given in
  Eq.~\eqref{eq:jdfact}.}  This means that $X$ and $X^\mathcal{S}$ are
\emph{exchangeable} under the assumption of a joint distribution that
is factorised in terms of $\mathcal{S}$.

We sample datasets using the permutation scheme described in
\citet{henelius:2014:peek}. A new permuted dataset
$D^\mathcal{S}=\left(X^\mathcal{S},C\right)$ is created by permuting
the data matrix of the dataset $D=\left(X,C\right)$ at random. The
permutation is defined by $m$ bijective permutation functions $\pi_j:
[n]\mapsto [n]$ sampled uniformly at random from the set of allowed
permutations functions. The new data matrix is then given by
$X^\mathcal{S}\left(i,j\right)=X
\left(\pi_j\left(i\right),j\right)$. The allowed permutation functions
satisfy the following constraints for all $i\in [n]$, $j,j'\in[m]$,
and $S\in\mathcal{S}$:
\begin{enumerate}
\item permutations are within-a class, i.e.,
  $C\left(i\right)=C\left(\pi_j\left(i\right)\right)$, and
  \item items within a group are permuted together, i.e., $j\in
    S\wedge j'\in S\implies
    \pi_j\left(i\right)=\pi_{j'}\left(i\right)$.
\end{enumerate}

Let $\mathcal{D}_\mathcal{S}$ be the set of datasets that can be
generated by the above permutation scheme using the grouping
$\mathcal{S}$. We note:
\begin{lemma}\label{lem:g1}
  Each invocation of the permutation scheme produces each of the
  datasets in $\mathcal{D}_\mathcal{S}$ with uniform probability.
\end{lemma}
\begin{lemma}\label{lem:g2}
  The datasets in $\mathcal{D}_\mathcal{S}$ have equal probability under the
  distribution of Eq.~\eqref{eq:jdfact}, parametrised by $\mathcal{S}$.
\end{lemma}
\begin{proof}
  The proofs follow directly from the definition of the permutation
  and the probability distribution of Eq.~\eqref{eq:jdfact}.
\end{proof}
\begin{definition}\emph{Confidence intervals}
\label{def:ci}
  Given a dataset $D$, a grouping $\mathcal{S}$, a classifier $f$ and
  an integer $R$, let $A = \set{T\left( D_1^\mathcal{S} \right),
    \ldots, T\left( D_R^\mathcal{S} \right)}$ be a vector of
  accuracies where the datasets $D_i^\mathcal{S}$ are obtained by the
  permutation parametrised by $\mathcal{S}$, and $T$ is as in
  Eq.~\eqref{eq:teststatistic}. The CI is the tuple $ C =
  \left(c_\textrm{lower}, c_\textrm{upper} \right)$, where
  $c_\textrm{lower}$ and $c_\textrm{upper}$ are values corresponding
  to the 5\% and 95\% quantiles in $A$, respectively.
\end{definition}
We cast the above discussion as a problem:
\begin{problem}
\label{prob:structuretest}
Given an observed dataset $D$, a grouping $\mathcal{S}$ and a
classifier $f$, let $a_0$ be the accuracy of $f$ (trained using the
original data) on the test set. Determine if the upper end of the CI
of Def.~\ref{def:ci} for the accuracy of a classifier trained using
factorised data is at least $a_0$.
\end{problem}
If the above condition is met, we conclude that the factorisation
correctly captures the structure of the data.

\subsection{Automatically Finding Groupings (ASTRID)}
\label{sec:automatically}
In the previous section we examined whether a \emph{particular
  grouping} $\mathcal{S}$ describes the structure of the data in terms
of the factorisation in Eq.~\eqref{eq:jdfact}. A natural step is now
to ask \emph{how to find the grouping best describing the associations
  in a dataset $D$}? Here we choose \emph{best} to be the grouping
$\mathcal{S}$ of (i) maximum cardinality such that (ii) a classifier
trained using data shuffled with $\mathcal{S}$ is indistinguishable in
terms of accuracy from a classifier trained using the original,
unfactorised data.

Finding the maximum cardinality grouping is motivated by the fact that
in this case there are no irrelevant interactions. Also, interpreting
attribute interactions in small groups is easier than in large
groups. The requirement on accuracy means that no essential
information is lost and in practice this means that the upper end of
the CI for the accuracy of the classifier $f$ trained using
$D^\mathcal{S}$ is at least as large as the original accuracy $a_0$ of
$f$ trained using $D$.

Exhaustive search of all groupings is in general impossible due to the
size of the search space. Hence, to make our problem tractable we
assume that accuracy decreases approximately monotonically with
respect to breaking of groups in the correct solution, i.e., the more
the interactions are broken, the more classification performance
decreases. Using this property we use a \emph{top-down greedy
  algorithm} termed \astrid. For details see the extended description
in \citet{henelius:2016:astrid}. In practice, $T$ in
Eq.~\eqref{eq:teststatistic} is susceptible to stochastic variation
and for stability we instead use \emph{expected accuracy} $V$ when
optimising accuracy in the greedy algorithm:
\begin{equation}\label{eq:that}
  V \left(\mathcal{S} \right)=\frac 1{N}\sum_{i=1}^{N}T\left(D^\mathcal{S}_i\right),
\end{equation}
where $N$ is the number of samples used to calculate the expectation,
$D^\mathcal{S}_i$ ($i\in [N]$) is a dataset generated by the
permutation parametrised by $\mathcal{S}$ and $T$ is defined as in
Eq.~\eqref{eq:teststatistic}.

\section{Experiments}
\label{sec:experiments}
We use \astrid to identify attribute interactions. We use a synthetic
dataset and 11 datasets from the UCI machine learning repository
\cite{bache:2014:a}\footnote{Datasets obtained from
  \url{http://www.cs.waikato.ac.nz/ml/weka/datasets.html}}. All
experiments were run in R \citep{R:2015:a} and our method is released
as the \astrid R-package, available for
download\footnote{\url{https://github.com/bwrc/astrid-r} (R-package
  and source code for experiments)}. We use a value of $R = 250$ in
Def.~\ref{def:ci} and $N= 100$ in Eq.~\eqref{eq:that}. In all
experiments the dataset was randomly split as follows: 50\% for
training ($D$) and the rest for testing ($D_{\mathrm{test}}$, see
Eq.~\eqref{eq:teststatistic}): 25\% for computing $V$
(Eq.~\eqref{eq:that}), and 25\% for computing CIs. As classifiers we
use support vector machines (SVM) with RBF kernel, random forest (RF)
and \nb (NB).

The datasets are summarised in Table~\ref{tab:datasets}. The
\textbf{UCI datasets} were chosen so that the SVM and random forest
classifiers achieve reasonably good accuracy at default settings,
since the goal here is to demonstrate the applicability of the method
rather than optimise classifier performance. Rows with missing values
and constant-value columns were removed from the UCI datasets. The
\textbf{synthetic dataset} has two classes, each with 500 data
points. Attributes $1$ and $2$ carry meaningful class information only
when considered jointly, attribute $3$ contains some class information
and attribute $4$ is random noise. The correct grouping is hence
$\mathcal{S} = \set{\set{1,2}, \set{3}, \set{4}}$.

\begin{table}[t!]
\setlength{\tabcolsep}{1.2ex} 
\centering
\caption{The datasets used in the experiments (2--10 from
  UCI). Columns as follows: Number of items (Ni) after removal of rows
  with missing values, number of classes (Nc) after removal of
  constant-value columns, number of attributes (Na). MCP is major
  class proportion. \textbf{T$_\mathrm{\textbf{SVM}}$} and
  \textbf{T$_\mathrm{\textbf{RF}}$} give the computation in minutes of the \astrid method
  for the SVM and random forest, respectively.}
\label{tab:datasets}
\tiny{
      \begin{tabularx}{\columnwidth}{clccccrr}
  \toprule
  \textbf{n} & \textbf{Dataset} & \textbf{Ni} & \textbf{Nc} & \textbf{Na} & \textbf{MCP} & \textbf{T$_\mathrm{\textbf{SVM}}$} & \textbf{T$_\mathrm{\textbf{RF}}$} \\ 
  \midrule
  1  & \texttt{synthetic} & 1000 & 2 & 4 & 0.50 & 0.1 & 0.4 \\
  2  & \texttt{balance-scale} & 625 & 3 & 4 & 0.46 & 0.1 & 0.3 \\ 
  3  & \texttt{diabetes} & 768 & 2 & 8 & 0.65 & 0.2 & 1.1 \\ 
  4  & \texttt{vowel} & 990 & 11 & 13 & 0.09 & 1.2 & 56.1 \\ 
  5  & \texttt{credit-a} & 653 & 2 & 15 & 0.55 & 0.8 & 3.5 \\ 
  6  & \texttt{vote} & 232 & 2 & 16 & 0.53 & 0.6 & 0.9 \\ 
  7  & \texttt{segment} & 2310 & 7 & 18 & 0.14 & 3.7 & 14.2 \\ 
  8  & \texttt{vehicle} & 846 & 4 & 18 & 0.26 & 1.5 & 6.8 \\ 
  9  & \texttt{mushroom} & 5644 & 2 & 21 & 0.62 & 13.1 & 19.8 \\ 
  10 & \texttt{soybean} & 682 & 19 & 35 & 0.13 & 9.1 & 29.5 \\ 
  11 & \texttt{kr-vs-kp} & 3196 & 2 & 36 & 0.52 & 42.2 & 41.7 \\ 
  \bottomrule
      \end{tabularx}
}
\end{table}


\section{Results}
\label{sec:results}
The results are presented as tables where each row is a grouping and
the columns represent attributes. Attributes belonging to the same
group are marked with the same letter, i.e., attributes marked with
the same letter on the same row are interacting.

Table~\ref{res:tab:synthetic} shows the results for the synthetic
dataset where the highest-cardinality grouping is highlighted and is
also shown below the table. Using the SVM and RF classifiers \astrid
identifies the correct attribute interaction structure ($k = 3$). For
$k=4$ the accuracy is clearly lower. For \nb all groupings (all values
of $k$) are equally valid since the classifier assumes attribute
independence. The results mean that the average accuracy of an SVM or
RF classifier trained on the synthetic dataset permuted using
$\mathcal{S} = \set{\set{1,2}, \set{3}, \set{4}}$ is within
CIs. \astrid reveals the factorised form of the joint distribution of
the data, which makes it possible to identify the attribute
interaction structure exploited by the classifier in the
datasets. This makes the models more interpretable and we, e.g., learn
that NB does not exploit interactions (as expected!).

The groupings for the UCI datasets are summarised in
Table~\ref{res:tab:uci}. SVM and RF are in general similar in terms of
the cardinality ($k$), with the exception of \texttt{kr-vs-kp} and
\texttt{soybean}. In many cases it appears that the classifiers
utilise few interactions in the UCI datasets. To compare this finding
with the results of \citet{ojala2010jmlr}, we calculated the value of
their Test~2, denoted $p_\mathrm{OG}$ in Table~\ref{res:tab:uci}. This
test investigates whether a classifier utilises attribute
interactions. $p_\mathrm{OG} \geq 0.05$ indicates that no attribute
interactions are used by the classifier, which we find for
\texttt{diabetes} and \texttt{soybean} for SVM and for
\texttt{diabetes} and \texttt{credit-a} for random forest (highlighted
in the table). This is in line with the findings from \astrid, since
for these datasets $k$ equals $N$ in Table~\ref{res:tab:uci} and no
interactions are hence utilised as the dataset can be factorised into
singleton groups.

Finally, as an illustrative example of grouping attributes exploited
by a classifier we consider the \texttt{vote} dataset. This dataset
contains yes/no information on 16 issues with the target of
classifying if a person is republican or democrat. Using SVM \astrid
finds that the maximum cardinality grouping is of size $k = 8$
(Tab.~\ref{res:tab:uci}). The grouping consists of {\bf 7 singleton
  attributes} (\attr{water\hyp project\hyp cost\hyp sharing},
\attr{synfuels\hyp corporation\hyp cutback}, \attr{physician\hyp
  fee\hyp freeze}, \attr{education\hyp spending}, \attr{duty\hyp
  free\hyp exports}, \attr{export\hyp administration\hyp act\hyp
  south\hyp africa}, \attr{immigration}) and {\bf one group with 9
  interacting attributes} (\attr{crime}, \attr{handicapped\hyp
  infants}, \attr{religious\hyp groups\hyp in\hyp school},
\attr{superfund\hyp right\hyp to\hyp sue}, \attr{adoption\hyp of\hyp
  the\hyp budget\hyp resolution}, \attr{mx\hyp missile},
\attr{anti\hyp satellite\hyp test\hyp ban}, \attr{aid\hyp to\hyp
  nicaraguan\hyp contras}, \attr{el\hyp salvador\hyp aid}). It appears
that the 9 attributes in the group roughly represent military and
foreign policy issues, and economic and social issues. This means,
that the SVM exploits relations between these 9 political issues when
classifying persons into republicans or democrats. On the other hand,
the singleton attributes seem to mostly represent domestic economic,
economic and export issues. The classifier does not use any singleton
attribute jointly with any other attribute when making predictions.

Note that \astrid is a randomised algorithm and the found groupings
are hence not necessarily unique. The stability of the results depends
on factors such as the used classifier, the size of the data and the
strength of the interactions. Also, the results are affected by the
number of random samples ($R$ in Def.~\ref{def:ci} and $N$ in
Eq.~\eqref{eq:that}) and for practical applications a trade-off
between accuracy and speed must be made.

\begin{table}[t!]
\setlength{\tabcolsep}{0.2ex} 
  \centering
\caption{The \texttt{synthetic} dataset. The cardinality of the
  grouping is $k$ and CI is the confidence interval for
  accuracy. Original accuracy using unshuffled data ($a_0$) and the
  final grouping ($\mathcal{S}$, highlighted row) shown above and
  below the table, respectively. An asterisk ($*$) denotes that the
  factorisation is valid.}
\label{res:tab:synthetic}
\begin{subtable}{0.49\columnwidth}
\centering
\caption{SVM}
\label{res:tab:synthetic:svm}
\scriptsize{
\begin{tabular}{ccccccc} 
& \multicolumn{6}{l}{$a_0 = 0.908$}\\
\toprule 
\textbf{k} & \textbf{CI} & & \rotatebox{90}{a3} & \rotatebox{90}{a4} & \rotatebox{90}{a2} & \rotatebox{90}{a1}\\ 
\cmidrule(){1-3} 
\cmidrule(l){4-7} 
2 & [0.900,\;0.920] & * & (A) & (B & B & B)\\ 
\rowcolor{cbluelight}
3 & [0.896,\;0.920] & * & (A) & (B) & (C & C)\\ 
4 & [0.696,\;0.784] &  & (A) & (B) & (C) & (D)\\ 
\bottomrule 
\end{tabular} 
}
$ \mathcal{S} = \set{\set{1,2}, \set{3}, \set{4}}$%
\end{subtable}%
\hspace*{2ex}%
\begin{subtable}{0.49\columnwidth}
\centering
\caption{Random forest}
\label{res:tab:synthetic:rf}
\scriptsize{
\begin{tabular}{ccccccc} 
& \multicolumn{6}{l}{$a_0 = 0.904$}\\
\toprule 
\textbf{k} & \textbf{CI} & & \rotatebox{90}{a3} & \rotatebox{90}{a4} & \rotatebox{90}{a1} & \rotatebox{90}{a2}\\ 
\cmidrule(){1-3} 
\cmidrule(l){4-7} 
2 & [0.896,\;0.928] & * & (A) & (B & B & B)\\ 
\rowcolor{cbluelight}
3 & [0.896,\;0.928] & * & (A) & (B) & (C & C)\\ 
4 & [0.668,\;0.756] &  & (A) & (B) & (C) & (D)\\ 
\bottomrule 
\end{tabular} 
  $\mathcal{S} = \set{\set{1,2}, \set{3}, \set{4}}$
}
\end{subtable}
\begin{subtable}{0.33\textwidth}
  \centering
\caption{Na\"ive Bayes}
\label{res:tab:synthetic:nb}
\scriptsize{
\begin{tabular}{ccccccc} 
& \multicolumn{6}{l}{$a_0 = 0.760$}\\
\toprule 
\textbf{k} & \textbf{CI} & & \rotatebox{90}{a1} & \rotatebox{90}{a2} & \rotatebox{90}{a3} & \rotatebox{90}{a4}\\ 
\cmidrule(){1-3} 
\cmidrule(l){4-7} 
2 & [0.760,\;0.760] & * & (A) & (B & B & B)\\ 
3 & [0.760,\;0.760] & * & (A) & (B) & (C & C)\\ 
\rowcolor{cbluelight}
4 & [0.760,\;0.760] & * & (A) & (B) & (C) & (D)\\ 
\bottomrule 
\end{tabular} 
  $\mathcal{S} = \set{\set{1}, \set{2}, \set{3}, \set{4}}$
}
\end{subtable}
\end{table}

\begin{table}[t!]
 \setlength{\tabcolsep}{0.6ex} 
  \centering
\caption{Groupings for UCI datasets. Columns as follows: number of
  attributes in the dataset (\emph{N}), size of the grouping
  (\emph{k}), size of the largest ($N_1$) and second-largest ($N_2$)
  groups, baseline accuracy for the classifier trained with unshuffled
  data (a${}_0$) and the CI. $p_\textrm{OG}$ is the $p$-value of Test
  2 in \citet{ojala2010jmlr} ($p \geq 0.05$ highlighted).}
\label{res:tab:uci} 
\small{
    \begin{tabular}{lcc cccccl}
      \toprule

  \textbf{Dataset}
  & \multicolumn{1}{c}{\textbf{N}}
&
    & \multicolumn{1}{c}{\textbf{k}}
    & \multicolumn{1}{c}{\textbf{N${}_1$}}
    & \multicolumn{1}{c}{\textbf{N$_2$}}
    & \multicolumn{1}{c}{\textbf{a${}_0$}}
    & \multicolumn{1}{c}{\textbf{CI}}
    & \multicolumn{1}{c}{\textbf{p}$_\textrm{OG}$}
\\ \midrule

& & \multicolumn{7}{c}{\textbf{SVM}}\\
\cmidrule(l){3-9}
\textbf{ balance-scale } &  4 & &3 & 2 & 1 & 0.891 & [0.821, 0.897] & 0.03\\
\textbf{ credit-a } &  15 & &12 & 4 & 1 & 0.871 & [0.847, 0.871] & 0.04\\
\textbf{ diabetes } &  8 & &8 & 1 & 1 & 0.714 & [0.688, 0.740] & \cellcolor{cbluelight} 0.59\\
\textbf{ kr-vs-kp } &  36 & &33 & 4 & 1 & 0.917 & [0.922, 0.924] & 0.00\\
\textbf{ mushroom } &  21 & &15 & 7 & 1 & 0.995 & [0.991, 0.995] & 0.00\\
\textbf{ segment } &  18 & &3 & 16 & 1 & 0.948 & [0.936, 0.948] & 0.00\\
\textbf{ soybean } &  35 & &35 & 1 & 1 & 0.844 & [0.820, 0.850] & \cellcolor{cbluelight} 0.26\\
\textbf{ vehicle } &  18 & &3 & 15 & 2 & 0.767 & [0.719, 0.781] & 0.00\\
\textbf{ vote } &  16 & &8 & 9 & 1 & 0.931 & [0.897, 0.931] & 0.00\\
\textbf{ vowel } &  13 & &3 & 11 & 1 & 0.806 & [0.760, 0.806] & 0.00\\

\multicolumn{9}{c}{}\\
& & \multicolumn{7}{c}{\textbf{random forest}}\\
\cmidrule(l){3-9}
\textbf{ balance-scale } &  4 & &3 & 2 & 1 & 0.821 & [0.731, 0.833] & 0.02\\
\textbf{ credit-a } &  15 & &15 & 1 & 1 & 0.877 & [0.847, 0.883] & \cellcolor{cbluelight} 0.19\\
\textbf{ diabetes } &  8 & &8 & 1 & 1 & 0.703 & [0.698, 0.740] & \cellcolor{cbluelight} 0.89\\
\textbf{ kr-vs-kp } &  36 & &16 & 21 & 1 & 0.982 & [0.972, 0.982] & 0.00\\
\textbf{ mushroom } &  21 & &14 & 8 & 1 & 1.000 & [0.996, 1.000] & 0.00\\
\textbf{ segment } &  18 & &4 & 15 & 1 & 0.986 & [0.979, 0.986] & 0.00\\
\textbf{ soybean } &  35 & &24 & 12 & 1 & 0.964 & [0.946, 0.964] & 0.00\\
\textbf{ vehicle } &  18 & &3 & 13 & 4 & 0.752 & [0.710, 0.757] & 0.00\\
\textbf{ vote } &  16 & &10 & 7 & 1 & 0.948 & [0.897, 0.948] & 0.00\\
\textbf{ vowel } &  13 & &3 & 11 & 1 & 0.917 & [0.901, 0.917] & 0.00\\

  \bottomrule
    \end{tabular}
}
\end{table}



\section{Discussion and Conclusion}
\label{sec:discussion}
Interpreting black box machine learning models is an important
emerging topic in data mining and in this paper we present the \astrid
method for investigating classifiers. This method provides insight
into generic, opaque classifier by revealing how the attributes are
interacting. \astrid automatically finds in polynomial time the
maximum cardinality grouping such that the accuracy of a classifier
trained using the factorised data cannot be distinguished (in terms of
confidence intervals) from a classifier trained using the original
data. The method makes no assumptions on the data distribution or the
used classifier and hence has high generic applicability to different
datasets and problems. This work extends previous research
\citep{henelius:2014:peek, ojala2010jmlr} on studying attribute
interactions in opaque classifiers.

Knowledge of attribute interactions exploited by classifiers is
important in, e.g., pharmacovigilance and bioinformatics (see
Sec.~\ref{sec:introduction}) where powerful classifiers are used in
data analysis, since they make it possible to simultaneously
investigate multiple attributes instead of, e.g., just pairwise
interactions. Here \astrid allows the practitioner to automatically
discover attribute groupings, providing insight into the data by
making the classifiers more transparent.

\section*{Acknowledgements}
This work was supported by Academy of Finland (decision 288814) and
Tekes (Revolution of Knowledge Work project).

\bibliography{structure_identification}
\bibliographystyle{icml2017}

\end{document}